%% file: colt2026-sample.tex
\documentclass[final]{colt2026} %

\usepackage{style}

\title[Language Identification with Succinct Machine-Independent Traces]{Language Identification with Succinct Machine-Independent Traces}
\usepackage{times}

\coltauthor{%
 \Name{Moses Charikar} \Email{moses@cs.stanford.edu}\\
 \addr Stanford University
 \AND
 \Name{Jon Kleinberg} \Email{kleinberg@cornell.edu}\\
 \addr Cornell University%
 \AND
 \Name{Chirag Pabbaraju} \Email{cpabbara@cs.stanford.edu}\\
 \addr Stanford University%
}

\begin{document}

\maketitle

\begin{abstract}
Motivated by the power of large language models, there has been renewed interest in the Gold-Angluin model of language identification in the limit, with an eye toward variants of the model that might overcome the negative results for its original formulation. Recent papers on this question have proposed looking at {\em computational traces} and {\em annotations} of training strings as a source of additional power for a learner, reflecting empirical regularities such as the way that commented source code is easier to learn from than arbitrary source code, and text annotated with algorithmically generated chain-of-thought tokens can be easier to learn from than the raw text itself. This recent work has shown positive results for language identification in the presence of such computational traces, but the traces in these positive results come from explicit automata-theoretic machine models that generate the language, where the underlying vocabulary of tokens for the traces is very large. In this paper, we address two fundamental issues left open by this line of work: can we achieve positive results with traces that use only a small alphabet, and can we define traces directly from the language itself, without requiring an underlying machine model that generates it? We establish positive results for both of these questions: for an arbitrary collection of languages, we show how to define computational traces that enable identification in the limit, using an alphabet of tokens that is linear in the size of the alphabet that the languages are defined over, and independent of any other properties of the languages. 
\end{abstract}


\input{introduction}

\input{results}
\input{related}

\input{background}
\input{main-coloring}

\input{state-machines}

\acks{This work was supported by Moses Charikar’s and Gregory Valiant’s Simons Investigator Awards, a Google PhD Fellowship, a Simons Collaboration grant and a grant from the MacArthur Foundation.}

\bibliography{references}

\appendix
\crefalias{section}{appendix} %

\input{appendix}

\end{document}

%% file: introduction.tex
\section{Introduction}
\label{sec:introduction}

The success of large language models has motivated several lines of recent work exploring theoretical models that might provide insight into how they are achieving strong performance from the types of training data and training pipelines that they are built from. One direction this research has taken has been to start from a classical model of Gold on {\em language identification in the limit} \citep{gold1967language}, in which we view the language learning task as a game played between an adversary and an algorithm: the adversary thinks of a secret language $K$ known only to come from a countable list of candidates $L_1, L_2, L_3, ...$; the adversary enumerates the strings of $K$ in an arbitrary order\footnote{A valid enumeration $x_1,x_2,\dots$ of $K$ satisfies: (1) $x_i \in K$, for all $i$, and (2) $\forall x \in K$, there exists $i$ such that $x_i=x$.}; in each step $t$ the algorithm guesses the index $i_t$ of the true language; and the algorithm wins the game (it has {\em identified $K$ in the limit}) if for some time $t^*$, and all $t \geq t^*$, we have $L_{i_t} = K$ --- the algorithm is correct in every step beginning with $t^*$. 

The main results in Gold's framework, beginning with Gold's original theorem about his model, are negative \citep{gold1967language}; for all but highly constrained families of languages (much more restricted than the class of regular languages), the algorithm cannot win this game \citep{angluin1980inductive}. Yet real language models seem much more powerful than what these sweeping negative results suggest, and so the question has turned to finding reasonable variations on this model in which we are able to obtain positive results more consistent with the successes we see in practice. 

\paragraph{Computational traces.} One promising line of recent work has explored the power of {\em annotation} or {\em computational traces} in helping solve the problem of language identification in the limit \citep{papazov25learning,bhattamishra2026automata,peng2026language}. These models are motivated by the observation that training corpora often contains text or sequence data that is enriched by some kind of meta-data, thinking traces, or domain-relevant annotation. This can include cases in which the data itself has this type of supporting information, as in the way that learning from corpora of source code can be more effective if the code is commented, or learning from corpora of mathematics can be more effective if the proofs contain step-by-step arguments. It also includes cases in which the data has been annotated algorithmically, as in the success of {\em chain-of-thought} and related methods \citep{wei2022chain}. This range of examples suggests that learners can be more effective on a corpus of training data if that data is annotated with some kind of computational trace. 

These recent models have begun from the assumption that there is an automata-theoretic underpinning for the candidate languages --- for example, that they are produced by an unknown finite-state machine, pushdown automaton, or Turing machine --- and that the annotation in question is produced from the computational trace of the machine itself. Prior work obtains possibility results with this approach. This line of work leaves open two fundamental classes of questions, however. First, the kind of annotation they describe is fine-grained and expansive: in one case it is represented in terms of the states of the underlying machine \citep{peng2026language}, and in the other case it is represented by the power-set of the vocabulary, with an exponential blow-up as a result \citep{bhattamishra2026automata}. Second, the premise of these models is that annotation is directly tied to the automaton that produced the language; it leaves open the question of how we might usefully annotate for identification when we have no access to the underlying machine models. 

In contrast, the kinds of annotations or computational trace we see in practice tend not to come from an underlying machine model, but tend to be more designed as objects in their own right, with their own vocabularies, and much coarser than the full state set of the machine architecture. For example, in commented code, it is the comments that form the additional tokens helping learn the corpus; and with chain-of-thought, it is the thinking tokens that provide this help. Neither of these corresponds to the state of the underlying machine that produced the training data, to the extent that there is such a machine at all. What would a theory of computational traces look like if we take the tokens in the trace to have this coarse-grained structure, and independent of any specific machine model? Could we still obtain possibility results for language identification in the limit? 

\paragraph{The present work: Compact, machine-independent traces.} In this paper, we develop such a theory, and we show that annotation with a small set of tokens --- independent of the size of the state set of the underlying machine --- is sufficient to achieve language identification in the limit. Moreover, our construction of the annotation works at the level of the language itself, and does not assume knowledge of any underlying machine producing the language. 

We will describe the model in detail in the subsequent sections, but we begin here with a high-level overview of it. 
As a starting point, let's consider how earlier machine-based models defined computational traces in the case of a finite automaton: as the automaton processed the string one symbol at a time, from left to right, the trace reported the state of the automaton at each symbol. This kind of trace has the property if $x = x_1 x_2 \ldots x_n$ is the input string, then the trace annotates each input symbol $x_i$ with an annotation symbol $c_i$ (corresponding to the state of the automaton at that point); and because $c_i$ is the state, it depends only on the prefix $x_1 x_2 \ldots x_i$ up to that point, and not any subsequent symbols. 

We use this as the defining property of our traces: we say that a {\em trace coloring function} $c$ over a finite set of colors $P$ (the ``palette'') is a function that maps finite strings to elements of $P$, and we define the {\em color trace} of a string $x = x_1 x_2 \ldots x_n$ to be the value of $c$ applied to each of its prefixes: 
\begin{align} \trace_c(x) := (c(\eps), c(x_{\le 1}), c(x_{\le 2}), \dots, c(x_{\le n})) 
\end{align} 
where $x_{\le i}$ denotes the prefix $x_1 x_2 \ldots x_i$ and $\eps$ denotes the empty string. A coloring of $x$ by the states of a finite automaton that produces it is one way to produce such a trace coloring, but our definition here makes clear that we can define this object as an annotation scheme for any language, even if we don't know the underlying machine model, or even if there isn't one. (This is the same sense in which comments or thinking tokens can annotate a string in a training corpus even though they don't correspond in any respect to the generative mechanism that originally produced the string.) 

\paragraph{Identification using traces.} We would like to achieve identification in the limit using these traces, and further, %
with small palettes,
so that we do not need to provide annotations using voluminous alphabets like fine-grained state sets. Formally, given an instance of language identification in the limit with candidate languages $L_1, L_2, L_3, \ldots$, we consider a process in which each language $L_i$ is first annotated by its own trace coloring function $c_{L_i}$, and then an adversary enumerates the strings $x \in K$, each as an annotated pair consisting of $x$ and its color trace, $(x, \trace_{c_{K}}(x))$. 

Is this sufficient to enable language identification in the limit, even with no underlying machine model? We show that it is: there is an algorithm $\mcA$ such that for any collection of candidate languages $L_1, L_2, L_3, \ldots$, there exists a choice of trace colorings $c_{L_1}, c_{L_2}, c_{L_3}, \ldots$ for which $\mcA$ is able to achieve identification in the limit when presented with the ordered pairs $(x, \trace_{c_{K}}(x))$ arising from an enumeration of the adversary's language $K$. 

Moreover, the trace colorings required can be constructed using very few colors: for problem instances with an alphabet of size $k$, there is an algorithm that achieves this using trace colorings with palettes of size at most $k + 1$. Even in the case where the languages come from an underlying machine model, this is a bound that is independent of the number of states of these machines, which must grow unboundedly for an infinite family of languages. As such, it is a much more resource-efficient construction of computational traces --- and one that applies much more broadly --- than the computational traces proposed by earlier approaches. 

We prove this result by establishing an exact combinatorial characterization of when a trace coloring permits identification in the limit, and we relax this to a related sufficient condition for identification that we show how to achieve. This sufficient condition has interesting combinatorial interpretations in its own right, and we show matching upper and lower bounds for this condition in the case of languages over a binary alphabet (when $k = 2$): our result shows how to achieve the required condition with $k + 1 = 3$ colors, and we prove that 2 colors are not sufficient. 

We explore further extensions of our result, including the fact that identification in the limit is still possible (using a larger palette) when the trace can be corrupted to a bounded extent; and the fact that when the candidate languages $L_1, L_2, L_3, \ldots$ are regular, we can achieve identification in the limit with trace coloring functions that use only two colors.

%% file: results.tex
\subsection{Main Results}
\label{sec:main-results}

Given a language collection\footnote{All languages in the collection are assumed to be non-empty.} $\mcC$, our objective is to associate every language $L$ in $\mcC$ with a suitable trace coloring function $c_L$, such that when every $x \in L$ is accompanied with $\trace_{c_L}(x)$, the collection $\mcC$ becomes identifiable in the limit. Towards this, as our first contribution, we derive an exact characterization of when identification in the limit is possible with color traces. Our characterization is structurally similar to Angluin's condition \citep{angluin1980inductive} for identification in the limit, but additionally takes into account the color traces seen in the input; the proof is given in \Cref{sec:characterization-appendix}.

\begin{restatable}[Characterization of Identification in the Limit with Color Traces]{theorem}{characterization}
    \label{thm:characterization-identification-with-traces}
    Let $\mcC$ be a countable language collection. Then, $\mcC$ is identifiable in the limit with a color trace given by the trace coloring functions $\{c_L\}_{L \in \mcC}$ if and only if for every language $L \in \mcC$, there exists a finite ``tell-tale'' subset $T_L \subseteq L$, such that for every language $L' \in \mcC$ that is a proper subset of $L$, either (1) $L'$ does not contain $T_L$, or (2) there exists $x \in L'$ such that $\trace_{c_{L'}}(x) \neq \trace_{c_{L}}(x)$.
\end{restatable}
The characterization above gives a precise objective %
for constructing trace coloring functions $\{c_L\}_{L \in \mcC}$. Perhaps surprisingly, our next result, which is one of the main structural results of the paper, shows that for every language collection, we can construct trace coloring functions that use a small palette, and have the stronger property that they \textit{always} satisfy requirement (2) above. %

\begin{restatable}[Coloring Lemma]{lemma}{MainColoringLemma}
    \label{lemma:main-coloring-upper-bound}
    Let $\mcC$ be a language collection %
    over a finite alphabet $\Sigma$ of size $k$. There exist trace coloring functions $\{c_L\}_{L \in \mcC}$ mapping to a palette $P$ of size $k+1$, that satisfy the following ``distinguishable coloring condition'': for every $L, L' \in \mcC$ that satisfy $L \subsetneq L'$, there exists $x \in L$ for which
        $\trace_{c_L}(x) \neq \trace_{c_{L'}}(x)$.
\end{restatable}
We also show that our trace coloring functions use an optimal palette size for the purposes of satisfying the distinguishable coloring condition over a binary alphabet (\Cref{prop:2-coloring-lower-bound-for-sufficient-condition}): namely, there exists a (finite) language collection over a binary alphabet for which any set of trace coloring functions that satisfies the distinguishable coloring condition necessarily requires three colors.

Given that our trace coloring functions ensure that requirement (2) of the characterizing condition in \Cref{thm:characterization-identification-with-traces} always holds, 
tell-tale sets are not needed and we get our main result: 

\begin{restatable}[Identification in the Limit with $k+1$ Colors]{theorem}{MainTheorem}
    \label{thm:main-theorem}
    Let $\mcC$ be a countable language collection %
    over a finite alphabet $\Sigma$ of size $k$. There exist trace coloring functions $\{c_L\}_{L \in \mcC}$ mapping to a palette $P$ of size $k+1$ that make $\mcC$ identifiable in the limit with color traces.
\end{restatable}
In order to specify a trace coloring function for a language, in general, one needs to specify the color that it maps every string in the universe to. However, our trace-coloring functions from \Cref{thm:main-theorem} above have a particularly special structure for \textit{regular language} collections. Concretely, consider a regular language $L$, and let $M$ be any finite automaton that recognizes $L$. Denote its set of states by $Q(M)$. For any state $q$ in $Q(M)$, 
all strings that arrive at this state $q$ (starting from the initial state) are assigned the same color by our state coloring function $c_L$ for language $L$.
Now, associate the state $q$ with this common color.
Then, for any $x \in L$, the sequence of colors associated with the sequence of states it traverses in $M$ would exactly equal %
the color trace generated by our trace coloring function on $x$. Thus, for a regular language, our trace coloring functions $c_L$ can be described very compactly, by simply specifying the unique color associated with every state in any automaton recognizing $L$.

This structural property of our trace coloring functions motivates studying a natural class of trace coloring functions that come from directly coloring the \textit{states} of automata for regular languages --- we denote these as \textit{state coloring functions}. Concretely, a state coloring function $c_M$ maps every state in $Q(M)$ to a palette $P$. Denoting $L(M)$ to be the regular language recognized by $M$, every $x \in L(M)$ has a well-defined \textit{state color trace} associated with it: if $x$ traverses the sequence of states $q_0,\dots,q_{|x|}$ in $M$, then the state color trace is simply 
\begin{align}
    \label{eqn:state-color-trace-def}
    \trace_{c_M}(x)  := (c_{M}(q_0),\dots,c_{M}(q_{|x|})).%
\end{align}
The work of \citet{peng2026language} and \citet{bhattamishra2026automata} can be viewed in this framework (see Section~\ref{sec:related} for a discussion).
For the computational traces used by \citet{peng2026language}, the `color' of a state is the name of the state itself.
Our notion is a coarsening of this state information.

With this definition (\ref{eqn:state-color-trace-def}), %
we can now study the distinguishable coloring condition for any given collection $\mcC=\{M_1,M_2,\dots\}$ of finite automata. By the reasoning above, we can show that state coloring functions derived from our trace coloring functions satisfy the %
condition with $k+1$ colors. Recall that for arbitrary language collections, we could show that our trace coloring functions achieved an optimal palette size for the distinguishable coloring condition only in the case of a binary alphabet. Interestingly, for arbitrary  \textit{automaton} collections, we show that $k+1$ is the optimal palette size achievable by state coloring functions for the distinguishable coloring condition for any alphabet size $k$.

\begin{restatable}[State Coloring Lemma]{lemma}{StateColoringLemma}
    \label{lemma:state-coloring-optimal-bound}
    Let $\mcC$ be an automata collection %
    over an alphabet of size $k$. There exist state coloring functions $\{c_M\}_{M \in \mcC}$ mapping to a palette $P$ of size $k+1$, such that for any $M, M' \in \mcC$ that satisfy $L(M) \subsetneq L(M')$, there exists $x \in L(M)$ for which
        $\trace_{c_M}(x) \neq \trace_{c_{M'}}(x)$.
    Moreover, for any $k$, there exists a collection of automata over an alphabet of size $k$ such that any set of state coloring functions requires $k+1$ colors to satisfy this condition.
\end{restatable}

\Cref{lemma:state-coloring-optimal-bound} above is primarily about the tightness of $k+1$ colors for state coloring functions to satisfy the distinguishable coloring condition for any given collection of automata. We now return to our main question of identification in the limit with arbitrary trace coloring functions, but for the special class of regular language collections. In particular, we ask: what is the optimal palette size necessary for these collections? Instead of the general-purpose trace coloring functions given by \Cref{thm:main-theorem}, which use $k+1$ colors, can we perhaps use different trace coloring functions that use fewer colors?
Our final result shows that any collection of regular languages, where every language is infinite, is in fact identifiable in the limit using only \textit{two colors!} 

\begin{restatable}[Two Colors Suffice for Identifying Regular Languages]{theorem}{TwoColorsRegularLanguages}
    \label{thm:regular-languages-2-colors}
    Let $\mcC=\{L_1,L_2,\dots\}$ be a countable collection of regular languages, where every $L_i$ is infinite. Then, $\mcC$ is identifiable in the limit with color traces given by trace coloring functions $\{c_L\}_{L \in \mcC}$ that use just two colors.
\end{restatable}
We note that a palette size of two is optimal above, since using a single color is equivalent to having no trace information, and we know by Gold's results that there are simple regular language collections, where every language is infinite, that are not identifiable in the limit without any traces.

Finally, we also study identification in the limit with traces that are potentially corrupted (\Cref{sec:corruptions}); in this setting, we derive a \textit{robust} version of the coloring lemma, and analogously show optimality of the palette size in the case of state coloring functions for regular languages. 

We note that our primary focus in this paper is information-theoretic. Namely, for our upper bounds for identification, we seek to specify a function that maps inputs and color traces to a guess for the target language. Nevertheless, for implementing the identification algorithm of \Cref{thm:main-theorem}, we would require access to the coloring functions $c_L$ for every language $L$ in the collection, and the ability to determine membership of any string $x$ in any language $L$ in the collection.

%% file: related.tex
\subsection{Related Work}
\label{sec:related}
We discuss two recent works most closely related to our work, that demonstrate the power of annotation or computational traces for language identification.

\cite{peng2026language} take a machine-centric view, studying identification for languages accepted by a particular machine model (deterministic finite automata, deterministic push down automata, and Turing machines).
For a machine $M$ and input $x$, they consider $x$ augmented by a computational trace given by the sequence of computational steps of $M$ on $x$.
In the case of a DFA $D$, the computational trace is simply the sequence of states visited by $D$ while processing input $x$.
In the case of a Turing machine $M$, the computational trace is the sequence of `states' of the Turing machine while processing $x$, where each `state' includes the state of the finite state control, but also the contents of the memory tape, the position of the head, and so on. 
Given an enumeration of inputs accepted by machine $M$, augmented by their computational traces, they show that one can converge to the correct representation of $M$, and hence identify the language accepted by $M$.
We note that this notion of computational traces requires very detailed, fine-grained information to facilitate language identification.
The computational trace for input $x$ can be significantly longer than $x$.
They also study the question of robust identification, i.e., language identification from corrupted traces, establishing results in %
different error regimes.

Another recent work \citep{bhattamishra2026automata} also takes a machine-centric view, studying learnability of regular languages in the Next Symbol Prediction (NSP) setting. Here, the learner receives strings from the language, as well as rich annotation for every prefix of the string: whether the prefix belongs to the language, and which next symbols can lead to a string in the language, i.e., $k+1$ bits for alphabet size $k$.
They show that this annotation is information-theoretically sufficient to identify minimal DFAs; however the problem of improper learning remains computationally intractable.
We note that the annotation considered in this work is closely related to, but significantly richer that the annotation we use in our work. If the input alphabet is of size $k$, each prefix of the input string is annotated with a symbol from an alphabet size of $2^{k+1}$.
In contrast to our work, this work is focussed on regular languages.
Much of their focus is on computationally efficient learning -- an issue that we (mostly) do not study in our work.

%% file: background.tex
\section{Background and Motivating Examples}
\label{sec:background}

Language identification with a color trace can be related to language identification with a \textit{list} \citep*{charikar2025characterization}, where the algorithm guesses a list of $k$ languages at every step, and the algorithm identifies in the limit if the list eventually always contains the target language. In particular, \citet[Theorem 2]{charikar2025characterization} shows that a language collection $\mcC$ is identifiable in the limit with a list of size $k$ if and only if it can be expressed as a union of $k$ collections, \textit{each} of which is identifiable in the limit with a single guess. By this characterization, any collection $\mcC$ that is $k$-list identifiable decomposes as $\mcC = \cup_{i=1}^k \mcC_i$, where each $\mcC_i$ is identifiable. In this case, we can assign a constant color trace corresponding to a unique color $i$ to every language in $\mcC_i$. Then, the color traces seen along with the input reveal the specific $\mcC_i$ that the target language $K$ belongs to. Once this is known, %
one can simply apply any known algorithm that identifies in the limit to the collection $\mcC_i$; since $\mcC_i$ is identifiable in the limit, the algorithm will correctly identify the target language. Thus, $k$-list identifiability always implies identifiability with $k$ colors.

Interestingly, the converse is not true: identifiability with a color trace does not always imply identifiability with a finite-size list. To see this, fix any alphabet $\Sigma$. %
For any finite $F \subseteq \Sigma^*$, let $L_F=\Sigma^* \setminus F$. Now, consider the collection $\mcC = \{L_F: F \text{ is a finite subset of }\Sigma^*\}$. This collection is not identifiable with a list of any finite size (see Remark 3 in \citet{charikar2025characterization}). However, there is a simple trace coloring scheme with just two colors that makes the collection identifiable: associate every $L \in \mcC$ with the trace coloring function $c_L(x)=\Ind[x \in L]$. We call this the ``accept-reject'' coloring. Note then that $c_{L_F}(x)=0$ only if $x \in F$. Thus, an algorithm may initialize $F=\emptyset$, and initialize its guess for the target language $K$ to be $L_{\emptyset}$. Whenever it sees an input string $x$ whose color trace has $0$ in it, it obtains the prefix $w$ of $x$ at which the color $0$ was seen (i.e., $c_K(w)=0$). The algorithm then adds $w$ to $F$, and updates its guess $L_F$ accordingly. If the target language $K$ was $L_{\emptyset}$, the algorithm will never see a $0$ in any color trace, and its initial guess remains correct throughout. If the target language $K$ was $L_F$ for some $F \neq \emptyset$, then for every $w \in F$, the algorithm is guaranteed to see an input $x$ that contains $w$ as its prefix; at this point, the algorithm correctly infers $w$'s membership in $F$. Since $F$ is finite, the algorithm identifies $L_F$ in the limit.

The example above motivates studying the simple accept-reject coloring scheme further. Notice a key property about the collection in the example above: %
for every $L \in \mcC$, if a string $w$ was not in $L$, it could be extended by some suffix $y$ such that $wy \in L$. Indeed, %
this ``reject-extendability'' property is sufficient for a collection to be identifiable in the limit with the accept-reject coloring.

\begin{proposition}[Reject-Extendability Sufficient for Accept-Reject Coloring]
    \label{prop:reject-extendability-sufficiency}
    Let $\mcC$ be a countable language collection over the alphabet $\Sigma$, where every $L \in \mcC$ satisfies the property: for every $w \in \Sigma^*$, if $w \notin L$, then there exists a suffix $y \in \Sigma^*$ such that $wy \in L$. Then, the collection $\mcC$ is identifiable in the limit with a color trace given by the accept-reject trace coloring.
\end{proposition}
\begin{proof}
    Let $K \in \mcC=\{L_1,L_2,\dots\}$ be the target language. For every $L \in \mcC$, consider the accept-reject coloring given by the trace function $c_L(x)=\Ind[x \in L]$. Consider the algorithm, which at time step $t$, outputs the smallest index $i_t$ such that the language $L_{i_t}$ is consistent with all the information seen in the input so far. Namely, $L_{i_t}$ satisfies that $\{x_1,\dots,x_t\} \subseteq L_{i_t}$, and furthermore, $\trace_{c_K}(x_b)=\trace_{c_{L_{i_t}}}(x_b)$ for all $1 \le b \le t$. Furthermore, for every $j < i_t$, either $\{x_1,\dots,x_t\} \nsubseteq L_{j}$, or $\trace_{c_K}(x_b) \neq \trace_{c_{L_{j}}}(x_b)$ for some $1 \le b \le t$.
    
    Now, let $z$ be the smallest index in the collection for which $L_z=K$. We claim that the index $i_t$ output by the algorithm converges to $z$. %
    To see this, note that the index $z$ is always a valid candidate for the algorithm to output, since every string in the input is from $L_z=K$, and furthermore, %
    $c_{L_z}$ and $c_{K}$ are identical. Now, consider any $j < z$ for which $L_j \nsupseteq L_z$. For each such $L_j$, there exists some $x \in L_z \setminus L_j$ which shows up in the input at some finite time. At this time, the language $L_j$ becomes inconsistent. Since there are finitely many $j < z$, we have that beyond some finite time, every $L_j$ satisfying $j < z$ and $L_j \nsupseteq L_z$ is rendered inconsistent by the algorithm. So, we restrict our attention to languages $L_j$ that satisfy $j < z$ and $L_j \supsetneq L_z$ (note that $L_j$ cannot equal $L_z$, by our choice of $z$). Fix any such $L_j$, and consider any $w \in L_j \setminus L_z$. By the assumed reject-extendability property, there exists some suffix $y \in \Sigma^*$ such that $wy \in L_z$; furthermore, this string $wy$ is guaranteed to show up in the input at some finite time. But now, notice that $c_{K}(w)=c_{L_z}(w)=0$, but $c_{L_j}(w)=1$. Thus, $\trace_{c_{L_j}}(wy) \neq \trace_{c_{K}}(wy)$, and hence, the algorithm will declare $L_j$ to have an inconsistent color trace on $wy$. In this manner, every proper superset of $L_z$ occurring before it gets invalidated at some finite time, and we conclude that the algorithm's guess eventually converges to $z$.
\end{proof}
The sufficient condition above for which the accept-reject coloring works may be deemed to be fairly weak: it only stipulates that every rejected string may be extended in some way to acceptance. Could it not be necessary? Namely, could every language collection be identifiable with the simple accept-reject coloring? The following example shows that this is not the case.

\begin{example}[Accept-Reject Coloring Doesn't Always Work]
    \label{example:accept-reject-failure}
    Consider a collection $\mcC$ over the binary alphabet $\Sigma=\{0,1\}$, which comprises of the regular language $L_\infty=\{0,1\}^*$, together with the regular languages $L_i=\{0,1\}^* \setminus 1^i\{0,1\}^*$ for $i \ge 1$. Namely, $L_i$ excludes strings in $\{0,1\}^*$ that start with %
    $i$ ones. We have that for every $i \ge 1$, $L_i \subsetneq L_{i+1}$, $L_i \subsetneq L_\infty$, and furthermore, $\cup_i L_i = L_\infty$. Observe that for any $i \in \N \cup \{\infty\}$ and $x \in L_i$, the accept-reject coloring satisfies 
        $\trace_{c_{L_i}}(x)=\underbrace{(1,1,1,\dots,1)}_{|x|+1 \text{ times}}$.
    This is because each $L_i$ is \textit{prefix-closed}: every prefix of $x \in L_i$ is also contained in $L_i$. To see this, note that any $x \in L_i$ does not start with $1^i$, and hence no prefix of $x$ starts with $1^i$ as well. In particular, no string ever has 0 in its color trace. Furthermore, for any $i > j$, $\trace_{c_{L_i}}(x)=\trace_{c_{L_j}}(x)$ for every $x \in L_i$.
    
    Now, fix any identification algorithm, and consider an adversary that starts enumerating the strings in $L_1$ in the order that they appear in a fixed background ordering of $L_\infty$, together with their color traces (according to $c_{L_1}$, which is always the constant-ones trace). Then, at some finite time, the algorithm must guess $L_1$ to be the target language. At this time, the adversary switches to enumerating the strings in $L_2$, also in the fixed background ordering of $L_\infty$, and starting from the leftmost string in $L_2$ that has not been enumerated as yet. The adversary continues giving color traces according to $L_2$, which they can legally do, since $L_1 \subsetneq L_2$, and the traces of $L_2$ and $L_1$ matched on all strings in $L_1$.  %
    Now, the algorithm must guess the target language to be $L_2$ at some subsequent finite time, at which time the adversary switches to $L_3$. The adversary repeats this switching process endlessly, enumerating each $L_i$ for a finite phase. As the adversary begins enumerating $L_i$ from the leftmost string that has not yet been enumerated in the fixed background ordering of $L_\infty$, the adversary eventually enumerates every string in $L_\infty$, together also with its correct color trace according to $c_{L_\infty}$. But the algorithm guesses an incorrect language infinitely often, and hence does not identify $L_\infty$ in the limit. %
\end{example}
The structure that enabled the adversarial strategy in \Cref{example:accept-reject-failure} distills precisely into the characterizing condition of \Cref{thm:characterization-identification-with-traces}. Indeed, one direction of the characterization is a generalization of the adversarial strategy employed above. In the other direction, we appropriately modify the algorithm used in \Cref{prop:reject-extendability-sufficiency} to account for trace discrepancies, and tell-tales associated with languages. 

%% file: main-coloring.tex
\section{The Coloring Lemma}
\label{sec:main-coloring}

In this section, we restate and prove our main coloring lemma. We also discuss an extension to the setting where the color traces observed have a bounded number of corruptions.
\MainColoringLemma*
\begin{proof}
    For any language $L$, define the set $\Pref(L)$ to be the set of all strings in $\Sigma^*$ that can be extended to form a string contained in $L$. Namely,
    \begin{align}
        \label{eqn:pref-def}
        \Pref(L) := \{x \in \Sigma^*: \exists t \in \Sigma^* \text{ such that } xt \in L\}.
    \end{align}
    Now, for any $L \in \mcC$, define $\IsAccept_L: \Sigma^* \to \{0,1\}$ and $\NextLive_L:\Sigma^* \to \{0,1,\dots,k\}$ as:
    \begin{align}
        &\IsAccept_L(x) := \Ind[x \in L], \label{eqn:IsAccept-def}\\
        &\NextLive_L(x) := |\{a \in \Sigma: xa \in \Pref(L)\}|. \label{eqn:NextLive-def}
    \end{align}
    In words, $\IsAccept_L(x)$ is the indicator for $x$ being in $L$, whereas $\NextLive_L(x)$ counts the number of letters $a \in \Sigma$ such that $xa$ can be further extended to form a string contained in $L$. Now fix the palette $\palette=\{1,2,\dots,k+1\}$, and define the trace coloring function $c_L:\Sigma^* \to \palette$ as follows:
    
    \begin{align}
        \label{eqn:next-live-trace-coloring-function-def}
        c_L(x) = \begin{cases}
            \IsAccept_L(x) + \NextLive_L(x) & \text{if } x \in \Pref(L),\\
            1 & \text{otherwise.} \\
        \end{cases}
    \end{align}
    If $x \notin \Pref(L)$, $c_L(x)=1$; otherwise, if $x \in \Pref(L)$, then either $\IsAccept_L(x)=1$, or $\NextLive_L(x) > 0$; hence $c_L(x) \ge 1$. Therefore, $c_L(x) \in \{1,2,\dots,k+1\}$ for all $x \in \Sigma^*$. We now observe that the functions $c_L$ satisfy monotonicity on any string $x$ with respect to inclusion. %

    \begin{observation}[Monotonicity of Trace Coloring]
        \label{observation:monotonicity-trace-coloring}
        For any $x \in \Sigma^*$, and $L \subseteq L'$, it holds that
        \begin{align*}
            \text{(1) }\IsAccept_L(x) \le \IsAccept_{L'}(x), \quad \text{(2) }\ \NextLive_L(x) \le \NextLive_{L'}(x), \quad \text{(3) }\ c_L(x) \le c_{L'}(x).
        \end{align*}
    \end{observation}
    Now, fix any $L, L' \in \mcC$ for which $L \subsetneq L'$, and fix any $z \in L' \setminus L$.

    \paragraph{Case 1: $z \in \Pref(L)$.} In this case, let $t \in \Sigma^*$ be such that $zt \in L$ (note that $t \neq \eps$, since $z \notin L$). We claim that $x := zt$ satisfies $\trace_{c_L}(x) \neq \trace_{c_{L'}}(x)$. To see this, note that $\IsAccept_L(z) = 0$, but $\IsAccept_{L'}(z) = 1$; together with \Cref{observation:monotonicity-trace-coloring}, this implies that $c_L(z) < c_{L'}(z)$.

    \paragraph{Case 2: $z \notin \Pref(L)$.} In this case, let $w \in \Sigma^*$ be the \textit{maximal} prefix of $z$ that satisfies $w \in \Pref(L)$, and let $z=wat$, where $a \in \Sigma$, and $t \in \Sigma^*$. Here, by maximal, we mean that $wa \notin \Pref(L)$. Note that $w$ can be $\eps$, but its existence is guaranteed since $L$ is non-empty. %

    Now, since $w \in \Pref(L)$, there exists $y \in \Sigma^*$ such that $wy \in L$. We claim that $x := wy$ satisfies that $\trace_{c_L}(x) \neq \trace_{c_{L'}}(x)$. To see this, observe that $wa \notin \Pref(L)$, since $w$ was chosen to be the maximal prefix of $z$ satisfying $w \in \Pref(L)$. However, $wa \in \Pref(L')$, since $z=wat \in L' \setminus L$. %
    Thus, $\NextLive_L(w) < \NextLive_{L'}(w)$. By \Cref{observation:monotonicity-trace-coloring}, this implies that $c_L(w) < c_{L'}(w)$.

    In both cases, we have shown the existence of an $x \in L$ satisfying $\trace_{c_L}(x) \neq \trace_{c_{L'}}(x)$.
\end{proof}
The next claim shows that at least for $k=2$, the bound achieved by our construction is tight.

\begin{proposition}[Coloring Lemma Tight for $k=2$]
    \label{prop:2-coloring-lower-bound-for-sufficient-condition}
    There exists a finite collection $\mcC$ of non-empty languages over the binary alphabet $\Sigma=\{0,1\}$, such that any family of trace coloring functions $\{c_L\}_{L \in \mcC}$ that satisfies the distinguishable coloring condition
    must necessarily use $3$ colors.
\end{proposition}
\begin{proof}
    Consider the finite collection $\mcC=\{L_1, L_2, L_3, L_4, L_5\}$, where
    \begin{align*}
        L_1 = \{\eps\},\quad L_2 = \{\eps, 0\},\quad L_3 = \{\eps,1\},\quad L_4 = \{\eps, 0, 1\},\quad L_5=\{\eps, 0,1,01\}.
    \end{align*}
    Fix any trace coloring functions $c_{L_1}, c_{L_2}, c_{L_3}, c_{L_4}, c_{L_5}$ that map to the palette $\{\Red,\Blue\}$, and satisfy the distinguishable coloring condition. %

    Since $L_1$ only contains $\eps$, it must hold that $c_{L_1}(\eps) \neq c_{L_2}(\eps)=c_{L_3}(\eps)=c_{L_4}(\eps)=c_{L_5}(\eps)$. Without loss of generality, let $c_{L_1}(\eps)=\Red$, so that $c_{L_2}(\eps)=c_{L_3}(\eps)=c_{L_4}(\eps)=c_{L_5}(\eps)=\Blue$.

    Then, for $L_2$, it must hold that $c_{L_2}(0) \neq c_{L_4}(0)=c_{L_5}(0)$. Similarly, for $L_3$, it must be the case that $c_{L_3}(1) \neq c_{L_4}(1)=c_{L_5}(1)$. Without loss of generality, suppose $c_{L_2}(0)=\Red$ and $c_{L_3}(1)=\Red$, so that $c_{L_4}(0)=c_{L_5}(0)=\Blue$, and $c_{L_4}(1)=c_{L_5}(1)=\Blue$.

    But now, consider the traces of $L_4$. By the previous assignments, we have that
    \begin{alignat*}{2}
        &\trace_{c_{L_4}}(\eps) = (c_{L_4}(\eps)) = (\Blue),\qquad  &&\trace_{c_{L_5}}(\eps) = (c_{L_5}(\eps)) = (\Blue),\\
        &\trace_{c_{L_4}}(0) = (c_{L_4}(\eps), c_{L_4}(0)) = (\Blue, \Blue),\qquad &&\trace_{c_{L_5}}(0) = (c_{L_5}(\eps), c_{L_5}(0)) = (\Blue, \Blue),\\
        &\trace_{c_{L_4}}(1) = (c_{L_4}(\eps), c_{L_4}(1)) = (\Blue, \Blue),\qquad &&\trace_{c_{L_5}}(1) = (c_{L_5}(\eps), c_{L_5}(1)) = (\Blue, \Blue).
    \end{alignat*}
    So, even though $L_4 \subsetneq L_5$, every $x \in L_4$ satisfies $\trace_{c_{L_4}}(x) = \trace_{c_{L_5}}(x)$. This contradicts the trace coloring functions satisfying the distinguishable coloring condition.
\end{proof}
We remark that the lower bound above does not, in general, preclude identifiability in the limit with trace coloring functions that use only 2 colors. Rather, it merely shows that the distinguishable coloring condition may not be achieved for all collections using 2 colors. Indeed, the collection above is finite, and every finite collection is identifiable without any color traces \citep{angluin1980inductive}.

\subsection{Extension to Corrupted Traces}
\label{sec:corruptions}

We also obtain the guarantee of \Cref{lemma:main-coloring-upper-bound} in the setting where the color traces seen may be \textit{corrupted}. Formally, we consider a notion of corrupted traces similar to \citet{peng2026language}, where the trace seen along with an input string $x$ is $\widetilde{\trace}_{c_K}(x)$, a corrupted version of $\trace_{c_K}(x)$. We restrict the nature of allowed corruptions as follows. For any $x$, we require $|\widetilde{\trace}_{c_K}(x)|=|\trace_{c_K}(x)|$ (i.e., the corrupted trace must have the same length as the uncorrupted trace), and that every element in $\widetilde{\trace}_{c_K}(x)$ belongs to the same palette $\palette$ used by $c_K$ (i.e., the corrupted trace cannot introduce colors outside the palette). With these restrictions, we consider traces containing a bounded number of corruptions, as measured by the \textit{Hamming distance} between the corrupted and uncorrupted traces. Namely, if $\trace_{c_K}(x) = (a_1,a_2,\dots,a_n)$ and $\widetilde{\trace}_{c_K}(x) = (b_1,b_2,\dots,b_n)$, we restrict $d(a,b)=|\{i: a_i \neq b_i, 1 \le i \le n\}| \le \ell$, where $\ell > 0$ is an a priori known corruption budget.

In this setting, we wish to correctly identify the target language $K$ in the limit for any target language $K$, and for any enumeration $(x_1, \widetilde{\trace}_{c_K}(x_1)), (x_2, \widetilde{\trace}_{c_K}(x_2)),\dots$ of $K$ with corrupted traces satisfying a corrupted budget $\ell$. From the proof of \Cref{thm:characterization-identification-with-traces}, it is straightforward to verify that a \textit{sufficient} condition for identifiability of a collection $\mcC$ in this setting is the following: for every $L \in \mcC$, there exists a finite tell-tale $T_L \subseteq L$, such that for every $L' \in \mcC$ satisfying $L' \subsetneq L$, either (1) $L'$ does not contain $T_L$, or (2) there exists $x \in L'$ such that $d\left(\trace_{c_{L'}}(x), \trace_{c_{L}}(x)\right) > 2\ell$.

Then, just like \Cref{lemma:main-coloring-upper-bound}, we seek trace coloring functions that always satisfy (2) above. This will guarantee identifiability with corrupted traces having at most $\ell$ corruptions. Indeed, the next lemma constructs trace coloring functions that always satisfy (2) with a palette of size $O(k^{2\ell+1})$.

\begin{restatable}[Robust Coloring Lemma]{lemma}{robustcoloring}
    \label{lemma:corruptions-coloring-upper-bound}
    Let $\mcC$ be a language collection comprising of non-empty languages over a finite alphabet $\Sigma$ of size $k$, and fix $\ell > 0$. There exist trace coloring functions $\{c_L\}_{L \in \mcC}$ mapping to a palette $P$ of size $\max(8\ell+4, 8k^{2\ell+1})$, such that for any $L, L' \in \mcC$ that satisfy $L \subsetneq L'$, and $L$ contains some string $x$ having length at least $2\ell$, there exists $x \in L$ for which
        $d\left(\trace_{c_L}(x), \trace_{c_{L'}}(x)\right) > 2\ell$.
\end{restatable}
The proof of this lemma is given in \Cref{sec:robust-coloring-appendix}. The high-level idea is to propagate the discrepancy at a single location implied by the construction in \Cref{lemma:main-coloring-upper-bound} to all the $2\ell$ locations preceding it.

%% file: state-machines.tex
\section{Traces for Regular Languages}
\label{sec:state-machines}

The trace coloring functions that we constructed in \Cref{sec:main-coloring} satisfy a special property for regular languages. Consider any regular language $L$, and let $M$ be any \textit{Deterministic Finite-State Automaton (DFA)}\footnote{We give a formal definition of a DFA in \Cref{sec:state-traces-appendix}.} that recognizes $L$. Let $Q(M)$ be its set of states, $\delta_M$ be the associated state transition function, $q_0(M)$ be its initial state and $F(M) \subseteq Q(M)$ its accepting states. For any state $q$ in $Q(M)$, let $W_q$ be the set of all the strings that arrive at state $q$, starting from the initial state. %

Recall the definition of our trace coloring function $c_L$ from \eqref{eqn:next-live-trace-coloring-function-def}. We claim that for any state $q$, for any $w, w' \in W_q$, $c_L(w)=c_L(w')$. To see this, note that since $w$ and $w'$ both end up at the state $q$, for any suffix $y$, $wy$ and $w'y$ also end up in the same state in $Q(M)$. This implies that $w \in \Pref(L) \iff w' \in \Pref(L)$, and also that $\IsAccept_L(w)=\IsAccept_L(w'), \NextLive_L(w)=\NextLive_L(w')$. We conclude that $c_L(w)=c_L(w')$.

We now restate and prove \Cref{lemma:state-coloring-optimal-bound}, which constructs state coloring functions using $k+1$ colors so as to satisfy the distinguishable coloring condition, and also show that this palette size is optimal.

\StateColoringLemma*
\begin{proof}
    For any state $q \in Q(M)$, we say that $q$ is a \textit{live} state if there exists a suffix $y \in \Sigma^*$ such that $\delta_{M}(q,y) \in F(M)$; otherwise, we call $q$ a \textit{dead} state. %
    Now, similar to the definitions in \eqref{eqn:IsAccept-def} and \eqref{eqn:NextLive-def}, define $\IsAccept_M:Q(M) \to \{0,1\}$ and $\NextLive_M:Q(M) \to \{0,1,\dots,k\}$ as:
    \begin{align*}
        \IsAccept_M(q) := \Ind[q \in F(M)], %
        \qquad
        \NextLive_M(q) := |\{a \in \Sigma: \delta_M(q,a) \text{ is live}\}|. %
    \end{align*}
    Then, %
    define the state coloring function $c_{M}:Q(M) \to P$ as follows:
    \begin{align}
        \label{eqn:state-coloring-def}
        c_{M}(q) = \begin{cases}
            \IsAccept_M(q) + \NextLive_{M}(q) & \text{if $q$ is live}, \\
            1 & \text{otherwise.}
        \end{cases}
    \end{align}
    We can then verify that for any $x = x_0x_1\dots x_n \in L(M)$ (where $x_0=\eps$) that traverses the states $q_0,q_1,\dots,q_{n}$ in $M$, and for any $q_i$, it holds that
        $c_{M}(q_i) = c_{L(M)}(x_{0}x_1\dots x_i)$,
    where $c_{L(M)}$ is \textit{exactly} the trace coloring function defined in \eqref{eqn:next-live-trace-coloring-function-def} in the proof of \Cref{lemma:main-coloring-upper-bound}. %
    The conclusion of that lemma implies that $\{c_M\}_{M \in \mcC}$ satisfy the required distinguishable coloring condition.

    We now show that $k+1$ colors are optimal. Fix $k \ge 1$, and consider $\mcC=\{M_0, M_1,\dots,M_k\}$. Each $M_i$ has just two states $q_{0}(M_i)$ and $q_1(M_i)$, where $q_0(M_i)$ is the initial state, and also the only accepting state, and $q_1(M_i)$ is a rejecting sink state. Namely, $\forall i \forall a:$ %
    $\delta_{M_i}(q_1(M_i), a)=q_1(M_i)$.
    
    Now, in $M_0$, we further set $\delta_{M_0}(q_0(M_0), a) = q_1(M_0)$ for every $a \in \Sigma$. So, it holds that $L(M_0)=\{\eps\}$. Next, for $i \in \{1,\dots,k\}$, we set
    \begin{align*}
        \delta_{M_i}(q_0(M_i), a) = \begin{cases}
            q_0(M_i) & \text{if $a \in \{1,\dots,i\}$},\\
            q_1(M_i) & \text{otherwise.}
        \end{cases}
    \end{align*}
    Thus, $L(M_i)=\{1,\dots,i\}^*$. Observe that $L(M_0) \subsetneq L(M_1) \subsetneq \dots \subsetneq L(M_k)$.
    
    Now, fix any state coloring functions $c_{M_0}, c_{M_1},\dots,c_{M_k}$ that use a palette of size $k$, and assume that they satisfy: for every $M_i, M_j \in \mcC$ with $L(M_i) \subsetneq L(M_j)$, there exists $x \in L(M_i)$ for which $\trace_{c_{M_i}}(x) \neq \trace_{c_{M_j}}(x)$. Fix any $i,j \in \{0,1,\dots,k\}$ satisfying $i < j$. Observe that any $x \in L(M_i)$ simply loops on the state $q_0(M_i)$, giving $\trace_{c_{M_i}}(x)=\underbrace{(c_{M_i}(q_0(M_i)),\dots,c_{M_i}(q_0(M_i)))}_{|x|+1\text{ times}}$. Thus, for the assumed guarantee to hold, it must be the case that for every $i < j$, $c_{M_i}(q_0(M_i)) \neq c_{M_j}(q_0(M_j))$. But this is not possible if the state coloring functions use only $k$ colors.%
\end{proof}

In \Cref{sec:corrupted-state-traces-appendix}, we extend the construction of the state coloring functions above to the setting with corrupted traces, similar to \Cref{sec:corruptions}. Here, we show that a palette of size $O(k^{2\ell+1})$ suffices to achieve a similar guarantee as \Cref{lemma:corruptions-coloring-upper-bound}, and that $k^{2\ell-1}$ colors are necessary in general. %

Finally, we restate and prove \Cref{thm:regular-languages-2-colors}.
\TwoColorsRegularLanguages*

Before stating the formal proof, we give some intuition, specifically since the requirement that every $L_i$ be infinite might appear counterintuitive at first glance for this positive result. The trace coloring functions that we build interestingly arise via state coloring functions. That is, we carefully choose a sequence of automata $M_1, M_2,\dots$ that accept $L_1,L_2,\dots$, prescribe state coloring functions for these automata, and then directly argue that they achieve the goal of identification with just two colors. This is in contrast to \Cref{lemma:state-coloring-optimal-bound}, where the automata were fixed and provided to us, and we were concerned with the (stronger) distinguishable coloring condition; here, we choose suitable automata as a tool for constructing trace coloring functions for identification.

While there are several automata that recognize $L_i$, we want to choose $M_i$ with a view to satisfying the following property: for each previous automaton $M_j$ where $j < i$ and $L_j$ is a superset of $L_i$, we can designate a particular state in $M_i$ with a particular color, which helps distinguish it from $M_j$. In order to choose such a designated state for every $j < i$, we need enough states (at least $i-1$) in $L_i$ to begin with. A sufficient condition to ensure this is that every $L_i$ is infinite (intuitively, since there are long enough strings in the language, these must cause a loop, which can be ``disentangled'' to enlarge the state space). This ensures that with just two colors, we can distinguish the color traces of every language $L_i$ from the color traces of every superset of $L_i$ that appears before $L_i$. Note that this is a weaker property than the distinguishable coloring condition (which requires such a property for \textit{all} supersets in the collection, not just the ones appearing before), but is sufficient for identification in the limit (see \Cref{remark:weaker-condition-sufficient}).

We now give the formal proof.

\begin{proof}
    Fix $M_1,M_2,\dots$ to be any arbitrary DFAs that recognize $L_1,L_2,\dots$. %
    Now recall the definition of a \textit{live} state in a DFA: a state $q \in Q(M)$ is live if $\exists$ a suffix $y \in \Sigma^*$ such that $\delta_M(q,y) \in F(M)$. Now call a state $q$ \textit{reachable} if $\exists w \in \Sigma^*$ such that $\delta_M(q_0(M), w)=q$. By \Cref{lemma:increasing-live-states} proved in \Cref{sec:state-traces-appendix}, we can assume that every $M_i$ has at least $i-1$ live and reachable states. We will now inductively determine a binary state coloring function $c_{M_i}$ for every $i \ge 1$. %

    To begin with, fix $c_{M_1}$ arbitrarily. Now, for any $i > 1$, let us assume that $c_{M_1},\dots,c_{M_{i-1}}$ have been determined. Consider $M_i$: there are at least $i-1$ live and reachable states in $M_i$. Pick any $i-1$ of these, and denote the set as $\{q_b\}_{b=1,\dots,i-1}$. Associate each $q_b$ with a prefix $w_b \in \Sigma^*$ and a suffix $y_b \in \Sigma^*$, such that $\delta_{M_i}(q_0(M_i), w_b)=q_b$, and $\delta_{M_i}(q_b, y_b) \in F(M_i)$, so that $x_b=w_by_b \in L_i$. Now, consider the set $\{j < i: L_j \supsetneq L_i\}$: this set has size at most $i-1$. Thus, we can map every $j$ in this set to a distinct live $q_b$ in $M_i$. Consider traversing the prefix $w_b$ associated with $q_b$ in $M_j$, and let $q'_b=\delta_{M_j}(q_0(M_j), w_b)$. We will set $c_{M_i}(q_b)=1-c_{M_j}(q'_b)$. After doing this for every $j$, we assign $c_{M_i}(q)$ for the remaining states in $M_i$ arbitrarily. This completes the specification of $c_{M_i}(q)$.

    Now consider any $i \ge 1$. We claim that: $(\star)$ for every $j < i$, if $L_{j} \supsetneq L_i$, then there exists $x$ in $L_i$ for which $\trace_{c_{M_i}}(x) \neq \trace_{c_{M_j}}(x)$. Indeed, fix any such $L_j$, and let $q_b$ be the distinct live state in $M_i$ that we had associated $L_j$ with in the inductive construction of $c_{M_i}$ above. Now consider the string $x_b=w_b y_b \in L_i$ associated with the state $q_b$ in $M_i$. Note that the prefix $w_b$ reaches the state $q'_b = \delta_{M_j}(q_0(M_j), w_b)$ in $M_j$, whereas it reaches the state $q_b$ in $M_i$; our construction above guarantees that $c_{M_i}(q_b)=1-c_{M_j}(q'_b)$. Thus, $\trace_{c_{M_i}}(x_b) \neq \trace_{c_{M_j}}(x_b)$, as claimed.
    
    Finally, by \Cref{remark:weaker-condition-sufficient} in the proof of \Cref{thm:characterization-identification-with-traces}, the property $(\star)$ is sufficient for the algorithm considered in that proof to identify the collection in the limit with traces.
\end{proof}

\section{Conclusion and Discussion}
\label{sec:conclusion}

We studied the problem of language identification with coarse-grained annotations, independent of any underlying machine implementation of the target language. Our main result shows that every countable language collection can be identified in the limit with annotations the same size as the input strings, 
with alphabet size one more than the input alphabet size.

One open problem is to prove a lower bound on the number of colors required for language identification with a color trace. For regular language collections, our result shows that two colors always suffice, independent of their alphabet size $k$. Do constantly many colors always suffice for any arbitrary collection? Or must the number of colors grow linearly with $k$?
Another open problem is to establish a lower bound on the number of colors required to achieve the distinguishable coloring condition --- going beyond our lower bound of 3 colors seems challenging.

Regarding the question about the number of colors necessary for identification, very recently, we have been able to show that the result of \Cref{thm:regular-languages-2-colors}, that trace coloring functions with just two colors suffice for identifying collections of regular languages that are all infinite, can in fact be generalized to collections of \textit{arbitrary} languages that are all infinite. In fact, while trace coloring functions color every character in each string, this strengthened result can even be obtained with ``terminal'' coloring functions, where one only needs to assign a color to the entire string at once.

However, the terminal coloring functions for each language in this 2-coloring result are \textit{collection-dependent} (as are the trace coloring functions from \Cref{thm:regular-languages-2-colors}). Note, in contrast, that our general-purpose trace-coloring functions from \Cref{thm:main-theorem} which use $k+1$ colors, are \textit{collection-independent}. This raises the natural question: do there exist collection-independent terminal coloring functions that use only two colors, and suffice for identification in the limit for all countable language collections (possibly with the restriction that every language is infinite)?

It turns out that the answer to this question is yes; however, the required 2-coloring functions turn out to be highly non-constructive (and, we can prove, necessarily so in a precise sense), involving reasoning with uncountable ordinals and transfinite recursion. As a result, this new result is in a sense not directly comparable to the general trace coloring result in \Cref{thm:main-theorem} of the present paper: the new result is quantitatively stronger (using 2 colors instead of $k+1$), but it markedly lacks the constructive, natural formulation of \Cref{thm:main-theorem}. Because the premise, techniques and analyses used to obtain these results largely depart from the present results in this paper, we defer them to a different manuscript \citep{charikar2026globally}. 

%% file: appendix.tex
\section{Omitted Proofs from \Cref{sec:main-coloring}}
\label{sec:omitted-proofs}

\subsection{Characterization of Identification in the Limit with Color Traces}
\label{sec:characterization-appendix}

\characterization*
\begin{proof}
    \paragraph{If direction.} Suppose that the stated condition holds for the collection $\mcC=\{L_1,L_2,\dots\}$ and the associated trace coloring functions $\{c_L\}_{L \in \mcC}$. Furthermore, let $T_L$ be the finite tell-tale subset associated with language $L$ in the collection. Denote the unknown target language by $K$, and consider the following identification algorithm: at time step $t$, upon seeing $(x_1,\trace_{c_K}(x_1)), \dots, (x_t,\trace_{c_K}(x_t))$, the algorithm outputs the smallest index $i_t$, which satisfies: (1) $\{x_1,\dots,x_t\} \subseteq L_{i_t}$, (2) $T_{L_{i_t}} \subseteq \{x_1,\dots,x_t\}$, and (3) $\trace_{c_{L_{i_t}}}(x_j) = \trace_{c_K}(x_j)$ for every $1 \le j \le t$. We remark that without the check in (3), the algorithm would be identical to Angluin's algorithm for identification in the limit \citep{angluin1980inductive}.

    Let $z$ be any index in the collection for which $L_z=K$, and $\trace_{c_{L_z}}(x)=\trace_{c_{K}}(x)$ for every $x \in K$. Note that such an index exists since $K \in \mcC$. %
    Note also that requirements (1) and (3) hold for index $z$ at all time steps. Furthermore, since $T_{L_z}$ is a finite subset of $K$, and the input eventually enumerates every string in $K$, there exists a finite time at which requirement (2) gets satisfied as well. Thus, beyond this time, $z$ is a valid candidate index for the algorithm to output. We proceed assuming that we are beyond this time.

    Now, consider any $j < z$ for which $L_j \nsupseteq L_z$. For each such $L_j$, there exists some $x \in L_z \setminus L_j$ which is guaranteed to show up in the input at some finite time. At this time and beyond, index $j$ violates (1). Since there are only finitely many $j < z$, we have that beyond some finite time, every $j$ satisfying $j < z$ and $L_j \nsupseteq L_z$ violates (1). So, we assume being beyond this time as well, and next restrict our attention to indices $j$ that satisfy $j < z$ and $L_j \supsetneq L_z$. 
    Suppose in one case that $T_{L_j} \subseteq L_z$ holds. Then, at some finite time, all of $T_{L_j}$ will show up in the input, and hence (2) will always be satisfied beyond this time by index $j$. But now observe that we are in a position where $L_z$ and $L_j$ satisfy: $L_z \subsetneq L_j$ and $T_{L_j} \subseteq L_z$. By the condition in the theorem, it must then hold that there exists $x \in L_z$ for which $\trace_{c_{L_z}}(x)=\trace_{c_K}(x) \neq \trace_{c_{L_j}}(x)$. Such an $x$ is guaranteed to show up in the input at some finite time, together with $\trace_{c_K}(x)$; at this time and beyond, index $j$ violates (3). In the other case, suppose that $T_{L_j} \nsubseteq L_z$. In this case, at least one element in $T_{L_j}$ will never show up in the input, and hence (2) will never be satisfied by index $j$. In all cases, we have shown that there exists a finite time beyond which index $j$ stops being a candidate for the algorithm. Since there are only finitely many $j < z$, we conclude that the algorithm eventually always outputs some $j \le z$ for which $L_j = L_z = K$.

    \begin{remark}[Weaker Condition Sufficient]
        \label{remark:weaker-condition-sufficient}
        We remark that the analysis of the algorithm requires only the following weaker condition on the the countable collection $\mcC=\{L_1,L_2,\dots\}$ and trace coloring functions $\{c_{L_j}\}_{j \in \N}$: for every language $L_j \in \mcC$, there exists a finite tell-tale subset $T_{L_j} \subseteq L_j$ such that for every language $L_i \in \mcC$ that is a proper subset of $L_j$ and $i > j$, either (1) $L_i$ does not contain $T_{L_j}$, or (2) there exists $x \in L_i$ such that $\trace_{c_{L_i}}(x) \neq \trace_{c_{L_j}}(x)$.
    \end{remark}

    \paragraph{Only if direction.} Fix a set of trace coloring functions $\{c_L\}_{L \in \mcC}$, and suppose that there exists a language $L \in \mcC$ such that for every finite $T \subseteq L$, there exists a language $L' \in \mcC$ that is a proper subset of $L$, such that $L'$ contains $T$  and also $\trace_{c_{L'}}(x)=\trace_{c_L}(x)$ for every $x \in L'$. We will show that the collection is not identifiable in the limit with color traces given by these trace coloring functions.

    Towards this, fix any identification algorithm, and consider an adversary that starts enumerating the strings in $L$, together with their color traces according to $c_{L}$, in the order that they appear in a fixed background ordering of $L$. Then, at some finite time, the algorithm must guess an index $i$ such that $L_i=L$. Suppose that the input strings enumerated so far by the adversary comprise of the set $T$. Then, by assumption, there exists a language $L' \in \mcC$ that is a proper subset of $L$, such that $L'$ contains $T$, and furthermore, $\trace_{c_{L'}}(x)=\trace_{c_L}(x)$ for every $x \in L'$ (including $T$). Thus, the adversary switches to enumerating strings in $L'$ with their color traces (which they can legally do, since $T \subseteq L'$, and the traces of $L'$ and $L$ match on all strings in $L'$). Now, the algorithm must guess an index $i$ satisfying $L_i=L'$ at some subsequent finite time. At this time, the adversary switches back to enumerating $L$ (they can do this since traces of $L'$ and $L$ match at all strings in $L'$), starting from the leftmost string that has not yet been enumerated in the fixed background ordering of $L$. Again, at some subsequent finite time, the algorithm must guess an index $i$ such that $L_i=L$. Suppose that the input strings enumerated so far by the adversary comprise of the set $T'$. Then, by assumption again, there exists a language $L'' \in \mcC$ that is a proper subset of $L$, such that $L''$ contains $T'$, and furthermore, $\trace_{c_{L''}}(x)=\trace_{c_L}(x)$ for every $x \in L''$ (including $T'$). Thus, the adversary switches to enumerating strings in $L''$. The adversary thus continues switching back and forth between $L$ and some other language $\tilde{L} \subsetneq L$ endlessly. Since at each switch back to $L$, the adversary begins with the leftmost string in $L$ that has not yet been enumerated in its fixed background ordering, the adversary eventually enumerates every string in $L$ along with its trace according to $c_L$. But also, each switch back to $L$ follows a time step at which the algorithm guesses an index $i$ for which $L_i = \tilde{L} \subsetneq L$, and hence the algorithm guesses an incorrect index infinitely often. Thus, the algorithm does not identify $L$ in the limit.
\end{proof}

\subsection{Robust Coloring Lemma}
\label{sec:robust-coloring-appendix}

\robustcoloring*
\begin{proof}
    For any language $L \in \mcC$, we extend the definitions of $\IsAccept_L$ and $\NextLive_L$ from the proof of \Cref{lemma:main-coloring-upper-bound}, and also consider some global quantities for the language. Namely, consider the functions $\IsAccept_L$ and $\NextLive_L$ that map $\Sigma^*$ to a non-negative integer, and the quantities $\ShortAccept_L$ and 
    $\ShortPrefix_L$, defined as follows:
    \begin{align}
        &\IsAccept_L(x) := |\{y \in \Sigma^{\{0,\dots,2\ell\}}: xy \in L\}|, \label{eqn:IsAccept-corruptions-def}\\
        &\NextLive_L(x) := |\{y \in \Sigma^{\{1,\dots,2\ell+1\}}: xy \in \Pref(L)\}|, \label{eqn:NextLive-corruptions-def} \\
        &\ShortAccept_L := |\{y \in \Sigma^{\{0,\dots,2\ell-1\}}: y \in L\}|, \label{eqn:ShortAccept-corruptions-def} \\
        &\ShortPrefix_L := |\{y \in \Sigma^{\{0,\dots,2\ell\}}: y \in \Pref(L)\}|. \label{eqn:ShortPrefix-corruptions-def}
    \end{align}
    Here, $\Sigma^{\{i,i+1,\dots,j}\}$ denotes the set $\{y \in \Sigma^*:i \le |y| \le j\}$, with $|y|=0$ implying the empty string $y=\eps$.

    In words, $\IsAccept_L(x)$ counts the number of suffixes $y$ having length $|y| \in \{0,1,\dots,2\ell\}$ for which $xy$ is in $L$, $\NextLive_L(x)$ counts the number of suffixes $y$ having length $|y| \in \{1,\dots,2\ell+1\}$ such that $xy$ can be extended further to form a string in $L$, $\ShortAccept_L$ counts the number of strings $y$ having length $|y| \in \{0,\dots,2\ell-1\}$ such that $y$ belongs to $L$, 
    and $\ShortPrefix_L$ counts the number of prefixes $y$ having length $|y| \in \{0,\dots,2\ell\}$ such that $y$ can be extended further to form a string in $L$. Note that $\ShortAccept_L$ and
    $\ShortPrefix_L$ are global properties of the language $L$. Furthermore, for any $x$,
    \begin{align}
        0 \le \IsAccept_L(x), \NextLive_L(x), \ShortAccept_L, \ShortPrefix_L \le \sum_{i=1}^{2\ell+1} k^i \le \max(2\ell+1, 2k^{2\ell+1}). \label{eqn:bound-palette-size-corruptions}
    \end{align}

    Now fix the palette $\palette=\{1,2,\dots,B\}$, where $B=\max(8\ell+4, 8k^{2\ell+1})$, and define the trace coloring function $c_L:\Sigma^* \to \palette$ as follows:
    \begin{align}
        \label{eqn:next-live-trace-coloring-function-corruptions-def}
        c_L(x) = \begin{cases}
            \IsAccept_L(x) + \NextLive_L(x) + \ShortAccept_L 
            + \ShortPrefix_L & \text{if } x \in \Pref(L),\\
            1 & \text{otherwise.} \\
        \end{cases}
    \end{align}

    If $x \notin \Pref(L)$, $c_L(x)=1$; otherwise, if $x \in \Pref(L)$, then either $\IsAccept_L(x) > 0$, or $\NextLive_L(x) > 0$, and hence $c_L(x) \ge 1$. Together with \eqref{eqn:bound-palette-size-corruptions}, this implies that $c_L(x) \in P$ for all $x \in \Sigma^*$. Similar to the proof of \Cref{lemma:main-coloring-upper-bound}, we now observe that the quantities defined above satisfy monotonicity on any string $x$ with respect to inclusion.

    \begin{observation}[Monotonicity of Trace Coloring]
        \label{observation:monotonicity-trace-coloring-corruptions}
        Let $L,L'$ be languages satisfying $L \subseteq L'$. Then, for any $x \in \Sigma^*$, it holds that:
        \begin{align*}
        &\text{(1) }\IsAccept_L(x) \le \IsAccept_{L'}(x), \quad \text{(2) }\NextLive_L(x) \le \NextLive_{L'}(x), \\
        &\text{(3) }\ShortAccept_L \le \ShortAccept_{L'}, \quad \text{(4) }\ShortPrefix_L \le \ShortPrefix_{L'}, \quad \text{(5) }c_L(x) \le c_{L'}(x).
        \end{align*}
    \end{observation}
     Now, fix any $L, L' \in \mcC$ for which $L \subsetneq L'$, and $L$ contains some string $x$ having length at least $2\ell$.

     \paragraph{Case 1: $\ShortAccept_L < \ShortAccept_{L'}$ or $\ShortPrefix_L < \ShortPrefix_{L'}$.} In this case, using \Cref{observation:monotonicity-trace-coloring-corruptions}, note that $c_L(x) < c_{L'}(x)$ for every $x \in L$. Then, choose any $x \in L$ having length at least $2\ell$, which is guaranteed to exist by assumption. We have that each of $c_L(\eps), c_L(x_{\le 1}),\dots, c_L(x)$ is strictly smaller than $c_{L'}(\eps), c_{L'}(x_{\le 1}),\dots, c_{L'}(x)$ respectively, and hence, $d\left(\trace_{c_L}(x), \trace_{c_{L'}}(x)\right) > 2\ell$ as required.

     \paragraph{Case 2: $\ShortAccept_L = \ShortAccept_{L'}$ and $\ShortPrefix_L = \ShortPrefix_{L'}$.} Since $L \subsetneq L'$, $y \in L \implies y \in L'$, and $y \in \Pref(L) \implies y \in \Pref(L')$. Therefore, this case implies the stronger properties that 
     \begin{align}
        &\{y \in \Sigma^{\{0,\dots,2\ell-1\}}: y \in L\}=\{y \in \Sigma^{\{0,\dots,2\ell-1\}}: y \in L'\}, \label{eqn:stronger-property-1}\\
        &\{y \in \Sigma^{\{0,\dots,2\ell\}}: y \in \Pref(L)\}=\{y \in \Sigma^{\{0,\dots,2\ell\}}: y \in \Pref(L')\}. \label{eqn:stronger-property-2}
    \end{align}
    Now, fix any $z \in L' \setminus L$. Note that by \eqref{eqn:stronger-property-1}, it must be the case that $|z| \ge 2\ell$. We then have the following two subcases:
    \paragraph{Subcase 2a: $z \in \Pref(L)$.} In this case, let $t \in \Sigma^*$ be such that $zt \in L$ (note that $t \neq \eps$, since $z \notin L$). We claim that $x := zt$ satisfies that $d\left(\trace_{c_L}(x), \trace_{c_{L'}}(x)\right) > 2\ell$. To see this, let $z=z_0z_1,\dots,z_n$, where $z_0=\eps$ and $n \ge 2\ell$. For any $p \le q$, denote $z_{p:q} := z_{p}z_{p+1}\dots,z_{q}$. We then observe that for any $i \in \{0,1,\dots,2\ell\}$, $\IsAccept_L(z_{0:n-i})$ is strictly smaller than $\IsAccept_{L'}(z_{0:n-i})$; this is because there exists a suffix $y=z_{n-i+1:n}$ having length $|y| \in \{0,\dots,2\ell\}$, for which $z_{0:n-i}y \notin L$ but $z_{0:n-i}y \in L'$ (note that $z_{0:n-i}y$ is simply equal to $z$). This implies that \\ $d\left(\trace_{c_L}(x), \trace_{c_{L'}}(x)\right) > 2\ell$.

    \paragraph{Subcase 2b: $z \notin \Pref(L)$.} In this case, let $w \in \Sigma^*$ be the \textit{maximal} prefix of $z$ that satisfies $w \in \Pref(L)$, and let $z=wat$, where $a \in \Sigma$, and $t \in \Sigma^*$. Here, by maximal, we mean that $wa \notin \Pref(L)$. We first claim that $|w| \ge 2\ell$. To see this, suppose that $|w| < 2\ell$. In this case, $|wa| \le 2\ell$. But it also holds that $wa \notin \Pref(L)$, and $wa \in \Pref(L')$ (since $z=wat \in L'$). This contradicts \eqref{eqn:stronger-property-2}. Thus, $|w| \ge 2\ell$.

    So, let $w=w_0w_1,\dots,w_n$, where $w_0=\eps$. We now argue that for every $i \in \{0,1,\dots,2\ell\}$, $\NextLive_L(w_{0:n-i})$ is strictly smaller than $\NextLive_{L'}(w_{0:n-i})$. Again, this holds because there exists a suffix $y=w_{n-i+1:n}a$ having length $|y| \in \{1,\dots,2\ell+1\}$, for which $w_{0:n-i}y \notin \Pref(L)$ but $w_{0:n-i}y \notin \Pref(L')$ (note that $w_{0:n-i}y$ is simply equal to $wa$). This then implies that \\ $d\left(\trace_{c_L}(x), \trace_{c_{L'}}(x)\right) > 2\ell$.

    In both cases, we have shown the existence of an $x \in L$ satisfying $d\left(\trace_{c_L}(x), \trace_{c_{L'}}(x)\right) > 2\ell$, as required.
\end{proof}

\section{Traces for Regular Languages}
\label{sec:state-traces-appendix}

We give a formal definition of a DFA following the textbooks \citet{sipser1996introduction,hopcroft2001introduction}:

\begin{definition}[Deterministic Finite-State Automaton (DFA]
    \label{def:dfa}
    A DFA $M$ over a finite alphabet $\Sigma$ is a 5-tuple $M = (Q(M), \Sigma, \delta_M, q_0(M), F(M))$, where $Q(M)$ is a finite set of states, $\delta_M : Q(M) \times \Sigma \to Q(M)$ is a state transition function, $q_0(M) \in Q(M)$ is the initial state, and $F(M) \subseteq Q(M)$ is the set of accepting states.  The regular language $L(M)$ that is accepted by the DFA $M$ is defined as the set $L(M)=\{x \in \Sigma^*:\delta_M(q_0(M),x) \in F(M)\}$.
\end{definition}
We can extend the domain of the transition function $\delta_M$ from $\Sigma$ to $\Sigma^*$ by defining $\delta_M(q,\eps)=q$, and recursively defining $\delta_M(q, a y) = \delta_M(\delta_M(q, a), y)$, for every $q \in Q(M), a \in \Sigma$ and $y \in \Sigma^*$.

\subsection{Corrupted State Traces}
\label{sec:corrupted-state-traces-appendix}

As we did in \Cref{sec:state-machines}, we describe how we can derive state coloring functions that are robust to corruptions in state traces from the trace coloring functions constructed in \Cref{lemma:corruptions-coloring-upper-bound} above. Consider any DFA $M$ in a given collection $\mcC=\{M_1,M_2,\dots\}$, and fix $\ell > 0$. For any $q \in Q(M)$, define 

\begin{align}
    &\IsAccept_M(q) := |\{y \in \Sigma^{\{0,\dots,2\ell\}}: \delta_M(q,y) \in F(M)\}|, \label{eqn:IsAccept-dfa-corruptions-def}\\
    &\NextLive_M(q) := |\{y \in \Sigma^{\{1,\dots,2\ell+1\}}: \delta_M(q,y) \text{ is live}\}|, \label{eqn:NextLive-dfa-corruptions-def} \\
    &\ShortAccept_M := |\{y \in \Sigma^{\{0,\dots,2\ell-1\}}: \delta_M(q_0(M),y) \in F(M)\}|, \label{eqn:ShortAccept-dfa-corruptions-def} \\
    &\ShortPrefix_M := |\{y \in \Sigma^{\{0,\dots,2\ell\}}: \delta_M(q_0(M),y) \text{ is live}\}|. \label{eqn:ShortPrefix-dfa-corruptions-def}
\end{align}
Note that $\ShortAccept_M$ and $\ShortPrefix_M$ are global properties of the DFA $M$.

Now fix the palette $\palette=\{1,2,\dots,B\}$, where $B=\max(8\ell+4, 8k^{2\ell+1})$, and define the state coloring function $c_M:Q(M) \to \palette$ as follows:
\begin{align}
    \label{eqn:next-live-state-coloring-function-corruptions-def}
    c_M(q) = \begin{cases}
        \IsAccept_M(q) + \NextLive_M(q) + \ShortAccept_M 
        + \ShortPrefix_M & \text{if $q$ is live},\\
        1 & \text{otherwise.} \\
    \end{cases}
\end{align}
For any $x = x_0x_1\dots x_n \in L(M)$ (where $x_0=\eps$), which traverses the states $q_0,q_1,\dots,q_{n}$ in $M$, we have that
\begin{align*}
    c_{M}(q_i) = c_{L(M)}(x_{0}x_1\dots x_i),
\end{align*}
where $c_{L(M)}$ is exactly the trace coloring function defined in \eqref{eqn:next-live-trace-coloring-function-corruptions-def} in the proof of \Cref{lemma:corruptions-coloring-upper-bound}. %
Thus, we have that for any $M,M' \in \mcC$ satisfying $L(M) \subsetneq L(M')$, there exists an $x \in L(M)$ for which $d\left(\trace_{c_M}(x), \trace_{c_{M'}}(x)\right) > 2\ell$, and we can identify $\mcC$ in the limit using state color traces having at most $\ell$ corruptions.

Analogous to \Cref{lemma:state-coloring-optimal-bound}, we next show that the palette size used above is optimal in order for state coloring functions in general.

\begin{claim}[$k^{2\ell-1}$ Colors Necessary for Corrupted State Traces]
    \label{claim:k^ell-lower-bound-dfas-corruptions}
    For any $k \ge 1$ and $\ell > 0$, there exists a finite collection $\mcC$ of DFAs over the alphabet $\Sigma=\{1,2,\dots,k\}$, such that any family of state coloring functions $\{c_M\}_{M \in \mcC}$ which satisfies: for every pair of DFAs $M, M' \in \mcC$ satisfying that $L(M) \subsetneq L(M')$, there exists $x \in L(M)$ for which $d\left(\trace_{c_M}(x),\trace_{c_{M'}}(x)\right) > 2\ell$, necessarily requires a palette of size $k^{2\ell-1}$.
\end{claim}
\begin{proof}
    Let $N=k^{2\ell-1}$, and $S=\{s_1,\dots,s_N\}$ be an ordering of all strings of size $2\ell-1$ in $\Sigma^{2\ell-1}$. Consider the collection $\mcC=\{M_1,\dots,M_N\}$, where each $M_i$ accepts the regular language
    \begin{align}
        L(M_i) = \{wy: w \in \{s_1,\dots,s_i\}, y \in \Sigma^*\}.
    \end{align}
    Namely, $L(M_i)$ has strings of the form $wy$, where $w$ is any of the first $i$ strings in $S$, and $y$ is any suffix in $\Sigma^*$. We have that $L(M_1) \subsetneq L(M_2) \subsetneq \dots \subsetneq L(M_N)$.
    
    Now observe that each $M_i$ can be specified by a DFA that first builds a prefix tree over the set $\{s_1,\dots,s_i\}$, followed by an accepting sink state. Namely, for any $x$, the prefix tree part checks whether the first $2\ell-1$ characters of $x$ belong to $\{s_1,\dots,s_i\}$---if they do, then the string is accepted. Let $q_0(M_i)$ be the initial state of $M_i$ (i.e., the root of the prefix tree), and let $q_F(M_i)$ be the single accepting sink state. Note then that any $x \in L(M_i)$ traverses a sequence of states of the form 
    \begin{align}
        \label{eqn:sequence-of-states-characterization-corruption-dfa-lb}
        \underbrace{q_{0}(M_i) \to q_{j_1}(M_i) \to \dots \to q_{j_{2\ell-1}}(M_i)}_{\text{prefix tree}} \to \underbrace{q_F(M_i) \to \dots \to q_F(M_i)}_{\text{loop on accepting sink state}}.
    \end{align}
    
    Now, fix any state coloring functions $c_{M_1},\dots,c_{M_N}$ that use a palette $P$ of size $N-1$, and satisfy that for every $M_i, M_j$ with $L(M_i) \subsetneq L(M_j)$, there exists $x \in L(M_i)$ for which \\$d\left(\trace_{c_{M_i}}(x), \trace_{c_{M_j}}(x)\right) > 2\ell$.

    Fix any $i,j \in \{1,\dots,N\}$ satisfying $i < j$. For any $x \in L(M_i)$, by the characterization \eqref{eqn:sequence-of-states-characterization-corruption-dfa-lb} of the sequence of states traversed by $x$ in $M_i$, we get that $\trace_{c_{M_i}}(x)$ is some sequence in $P^{2\ell}$, followed by the sequence $\underbrace{(c_{M_i}(q_F(M_i)),\dots,c_{M_i}(q_F(M_i)))}_{|x|-2\ell+1 \text{ times}}$. Similarly, we have that $\trace_{c_{M_j}}(x)$ is some sequence in $P^{2\ell}$, followed by the sequence $\underbrace{c_{M_j}(q_F(M_j),\dots,c_{M_j}(q_F(M_j)))}_{|x|-2\ell+1 \text{ times}}$. Thus, in order for $d\left(\trace_{c_{M_i}}(x), \trace_{c_{M_j}}(x)\right) > 2\ell$ to hold, we necessarily require that $c_{M_i}(q_F(M_i)) \neq c_{M_j}(q_F(M_j))$. But this is not possible for every $i < j$, if the state coloring functions use only $N-1$ colors.
\end{proof}

\subsection{Identifying Regular Languages with Two Colors}
\label{sec:regular-languages-two-colors}

\begin{lemma}[Increasing Live, Reachable States in DFA]
    \label{lemma:increasing-live-states}
    Let $M$ be a DFA satisfying $|L(M)|=\infty$, and let $M$ have $n$ states that are live and reachable. Then, for any $k > 0$, there exists a DFA $M'$ satisfying $L(M')=L(M)$, such that $M'$ has $n+k$ states that are live and reachable.
\end{lemma}
\begin{proof}
    Since $L(M)$ is infinite, it must contain strings of arbitrarily large length. Let $x=x_1x_2\dots x_t$ be any string in $L(M)$ having length $t \ge n$. Then, consider the sequence of states $q_0,q_1,\dots,q_t$ traversed by $x$ in $M$: note that every state in this sequence is live and reachable. Then, since there are only $n$ live and reachable states in $M$, and the number of states in this sequence is $t+1 > n$, by the pigeonhole principle, there exists at least one live and reachable state that appears twice in the sequence. Let $s$ be the smallest index for which there exists $i < s$ such that $q_i = q_s$ (i.e., this is the first time a state got repeated in the sequence), and consider the intermediate sequence of transitions $q_i \to q_{i+1} \to \dots \to q_{s-1} \to q_s$, together with the symbols $x_{i+1},\dots,x_{s-1},x_s$ that effected these transitions. Note that by the choice of $s$, the states $q_0,\dots,q_{s-1}$ are all distinct. %
    
    Now, consider a DFA $M'$ that operates on the same alphabet $\Sigma$. $M'$ has a copy $q'$ of every state $q$ in $M$, plus an additional state $q'_{\New}$. The state transitions for all the states in $M'$ are identical to their copies in $M$, except for the states $q'_{s-1}$ and $q'_{\New}$, where $q_{s-1}$ is the copy of the state $q_{s-1}$, and $q'_{\New}$ is the additional new state. The state transitions for $q'_{s-1}$ and $q_{\New}$ are as follows. First, we make the transitions out of $q'_\New$ to be identical to those out of $q_s$, i.e., set $\delta_{M'}(q'_{\New}, a)$ to be the copy of $\delta_{M}(q_s, a)$ for every $a \in \Sigma$. Next, we redirect the single transition $q'_{s-1} \xrightarrow{x_s} q'_s$ to be $q'_{s-1} \xrightarrow{x_s} q'_\New$ instead; namely, for any $a \neq x_s$, we set $\delta_{M'}(q'_{s-1}, a)$ to be the copy of the state $\delta_{M}(q_{s-1}, a)$, and we make $q'_{\New}$ reachable by setting $\delta_{M'}(q'_{s-1}, x_s)=q'_{\New}$. In essence, we have redirected the transition from $q'_{s-1}$ which was looping back into $q'_s$ to the new state $q'_{\New}$, whose outward transitions are then identical to $q_s$. The initial state in $M'$ is set to the copy of the initial state in $M$, and the accepting states in $M'$ are set to the copies of the accepting states in $M$, \textit{plus} the state $q'_{\New}$ if $q_s$ was an accepting state. This completes the specification of $M'$. 
    
    We now argue that $M'$ has $n+1$ live and reachable states, and that $L(M')=L(M)$; the lemma is then established by inducting this $k$ times. First, it is clear that for any state $q$ in $M$ that was either not live or reachable in $M$, its copy $q'$ has the same status in $M'$. Now, by construction, the newly added state $q'_{\New}$ and the state $q'_s$ are both live and reachable, as witnessed by the traversal of $x$ through $M'$, which passes through both $q'_s$ and $q'_\New$ on the way to an accepting state. %
    Finally, for every state $q \neq q_s$ that was live and reachable in $M$, its copy $q'$ remains live and reachable in $M'$. This can be seen by tracing out the sequence of states traversed by any $y \in L(M)$ that passed through $q$ on the way to an accepting state in $M$. In its traversal through $M'$, the only notable change is that a transition of the form $q_{s-1} \xrightarrow{x_s} q_s$ gets replaced by $q'_{s-1} \xrightarrow{x_s} q'_{\New}$. But since the transitions out of $q'_\New$ are identical to $q_s$, $y$ continues on its path to an accepting state, and any occurrences of $q$ in the traversal through $M$ overlap with occurrences of $q'$ in the traversal through $M'$; thus, $q'$ remains live and reachable in $M'$. Thus, we have argued that $M'$ has $n+1$ live and reachable states.

    It remains to argue that $L(M')=L(M)$. Again, this is made clear by tracing out the traversal of any $x$ through $M$ and $M'$: if the traversal never uses the single redirected transition $q'_{s-1} \xrightarrow{x_s} q'_{\New}$, then the end outcomes are clearly identical; if it does use this transition, then in $M$, the transition leads into the state $q_s$, whereas in $M'$, the transition leads into the state $q'_\New$; however, since the transitions out of $q'_\New$ and $q_s$ are identical, the end outcome is identical as well.
\end{proof}